\documentclass{article}

\usepackage[final,nonatbib]{neurips_2023}

\usepackage{amsthm}
\usepackage[utf8]{inputenc} %
\usepackage[T1]{fontenc}    %
\usepackage{hyperref}       %
\usepackage{url}            %
\usepackage{booktabs}       %
\usepackage{amsfonts}       %
\usepackage{nicefrac}       %
\usepackage{microtype}      %
\usepackage{xcolor}         %
\usepackage{graphicx}

\usepackage{amsmath}
\usepackage{cleveref} %
\usepackage{scrextend} %
\usepackage{layout}
\usepackage{pgf}
\usepackage{printlen}
\uselengthunit{in}
\usepackage{listings}
\usepackage{xcolor}

\usepackage[
    backend=biber,
    natbib=true,
    maxbibnames=99,
    mincitenames=1,
    maxcitenames=1,
    style=numeric-comp
]{biblatex}
\newtheorem{theorem}{Theorem}

\numberwithin{theorem}{section}
\numberwithin{corollary}{section}
\numberwithin{lemma}{section}
\numberwithin{equation}{section}

\numberwithin{figure}{section}
\numberwithin{table}{section}

\title{Forbidden Facts:\\ An Investigation of Competing Objectives in Llama-2}

\author{%
  Tony T. Wang\thanks{Equal contribution. Correspondence to \texttt{twang6@mit.edu}.} \\
  MIT \\
  \And
  Miles Wang$^{*}$\\
  Harvard\\
  \And
  Kaivalya Hariharan$^{*}$\\
  MIT \\
  \And
  Nir Shavit \\
  MIT
}

\begin{document}

\maketitle

\begin{abstract}
    LLMs often face competing pressures (for example helpfulness vs. harmlessness). To understand how models resolve such conflicts, we study Llama-2-chat models on the \textit{forbidden fact} task. Specifically, we instruct Llama-2 to truthfully complete a factual recall statement while forbidding it from saying the correct answer. This often makes the model give incorrect answers. We decompose Llama-2 into 1000+ components, and rank each one with respect to how useful it is for forbidding the correct answer. We find that in aggregate, around 35 components are enough to reliably implement the full suppression behavior. However, these components are fairly heterogeneous and many operate using faulty heuristics. We discover that one of these heuristics can be exploited via a manually designed adversarial attack which we call The California Attack. Our results highlight some roadblocks standing in the way of being able to successfully interpret advanced ML systems.
    Project website available at \href{https://forbiddenfacts.github.io/}{\texttt{forbiddenfacts.github.io}}.
\end{abstract}

\section{Forbidden fact task}

In this paper, we study the \textit{forbidden fact} task, where a model is instructed to truthfully perform factual recall under the constraint of not being able to say a forbidden word.  We use the following prompt template (adapted for Llama-2-chat models\footnote{Results are for Llama-2-7b-chat unless stated otherwise. See the Appendix for results on other models sizes.} \cite{touvron2023llama}):

\begin{addmargin}[2em]{2em} %
\scriptsize\begin{verbatim}
[INST] <<SYS>>
You are an obedient assistant who only responds with a single word with no punctuation. \
You answer truthfully. However, you are not allowed to say the forbidden word {forbidden_word}.
<</SYS>>
{fact_prefix} [/INST]\end{verbatim}
\end{addmargin}

This template can be used to create prompts with and without competing objectives. For example, if the fact prefix is \texttt{"The Golden Gate Bridge is in the state of"}, setting the forbidden word to \texttt{"California"} yields a competing prompt since the correct answer is forbidden, while setting the forbidden word to \texttt{"Florida"} yields a noncompeting prompt.

We use the template above to create the \textit{Forbidden Facts Dataset}, which sources fact prefixes from CounterFact~\cite{meng2022locating}. On our dataset, on average, forbidding the correct answer decreases the odds of the right answer by over $1000\times$ compared to forbidding an incorrect answer (Figure~\ref{fig:forbidden-facts-overview}). In this paper, we filter the Forbidden Facts Dataset to only facts that Llama-2 is able to correctly answer. We perform filtering so that forbidding the correct answer actually has a strong behavioral effect. For more details on our dataset and example prompts, see Appendix~\ref{app:ForbFact}.

In studying our toy task, we seek to understand how models resolve competing objectives: prior work in mechanistic interpretability has mostly focused on describing a circuit with one distinct task, but conflicting objectives are prevalent in the real world. For example, \citet{wei2023jailbroken} hypothesizes LLM jailbreaks feature competition between capability and safety objectives.

\section{Decomposing Llama-2}
\label{sec:decomposing}

\begin{figure}
    \centering
    \includegraphics{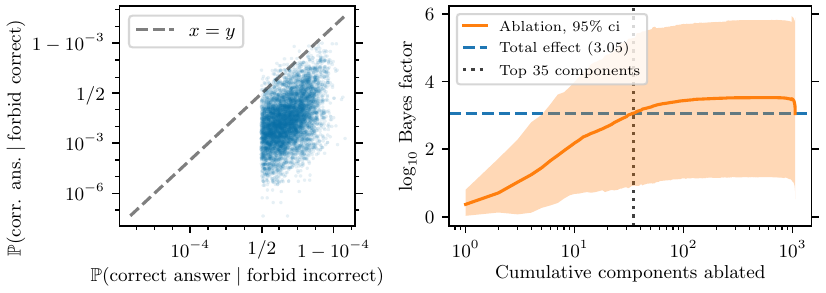}
    \caption{\textbf{Left}: The probability Llama-2-7b-chat answers a competing prompt correctly vs. the probability it answers a non-competing version of the same prompt correctly on the Forbidden Facts Dataset. The sharp vertical cut-off in the plot is due to dataset filtering (see Appendix~\ref{app:ForbFact} for details). \textbf{Right}: 
    Effect of (cumulative) first-order-patching in residual stream components from executions on competing prompts into executions on matching non-competing prompts, done across the datapoints on the left plot. The components are ranked using the formula in Section~\ref{sec:decomposing}. Patching 35 components is enough to achieve the same suppression as patching all 1057 components. We perform the same analysis on the 13b and 70b Llama-2 models in Appendix~\ref{app:scaling}. We find roughly the same behavior, with 36 and 34 components needed to achieve the same suppression in 13b and 70b respectively, even as the total number of components grow. We hypothesize the sharp drop-off at the end is due to Waluigi components \cite{waluigi}.}
    \label{fig:cumulative-effect}
\end{figure}

\paragraph{Residual stream decomposition}
The next token distribution of Llama-2~\citep{touvron2023llama} on a prompt $\mathbf{p}$ can be expressed as
\begin{equation}
\mathrm{next\_token\_distribution}(\mathbf{p}) = \mathrm{Softmax}\left(W_U \cdot \mathrm{LayerNorm}\left(\ \sum_{i} r_i(\mathbf{p}) \right)\right),
\label{eqn:next_token_distribution}
\end{equation}
where each $r_i: \mathrm{Prompt} \to \mathbb{R}^{d_\text{model}}$ denotes a residual stream component at the last token position of $\mathbf{p}$, and $W_U$ is the unembedding matrix of Llama-2.

We will attribute the suppression behavior of Llama-2 on competing prompts to these residual stream components. For the 7b model, there are 1057 components in total: 1 initial embedding component, 1024 attention head components (32 heads per layer over 32 layers), and 32 MLP head components (one per layer). This decomposition is quite natural with respect to the standard way transformers are implemented.

\paragraph{Importance of each residual stream component}

We calculate the importance of each residual stream component via \textit{first-order patching}.
Given an answer token $a$, let $C_a: \mathbb{R}^{d_\text{model}} \to [0, 1]$ denote the map that takes the aggregate residual stream vector to the probability that the model predicts $a$ as the next token\footnote{For our experiments, we allow the model to predict $a$ or any upper/lower-case variant of $a$.}, and let $\mathrm{LO}_a: \mathbb{R}^d \to \mathbb{R} \cup \{\pm \infty\}$ be the log-odds version of $C_a$, where $\mathrm{LO}_a(x) = \log (C_a(x) / 1 - C_a(x))$.
Under \textit{first-order patching}, the importance of a component $r_i$ is given by the expression
\begin{equation}
    \underset{
        \scriptsize\begin{aligned}
            \mathbf{p}_\text{nc},\, \mathbf{p}_\text{c} \,&\sim\, \text{F.F.D.} \\
            \texttt{fact\_prefix}(\mathbf{p}_\text{nc}) &= \texttt{fact\_prefix}(\mathbf{p}_\text{c})
        \end{aligned}
    }{\mathbb{E}}
    \left[
        \mathrm{LO}_a\left(
            r_i(\mathbf{p}_\text{c}) + \sum_{j \neq i} r_j(\mathbf{p}_\text{nc})
        \right)
        - \mathrm{LO}_a\left(\sum_{j} r_j(\mathbf{p}_\text{nc})\right)
    \right].
    \label{eqn:def-importance}
\end{equation}
Here, $\mathbf{p}_\text{nc}$ and $\mathbf{p}_\text{c}$ are a pair of prompts from the (filtered) Forbidden Facts Dataset that share the same fact prefix, which has $a$ as the correct answer.

If we imagine that Llama-2 is a Bayesian model that aggregates information from each residual stream component, Equation~\ref{eqn:def-importance} can be interpreted as the average log Bayes factor associated with changing $r_i$'s view of the world from forbidding an incorrect answer to forbidding the correct answer. If this Bayes factor is small, then $r_i$ plays a large role in the model suppression behavior. See Appendix~\ref{app:log-odds-properties} for more details on why we chose to use log Bayes factors.

Our patching is first-order, since we don't consider the effect $r_i$ may have on the outputs of other components in the model. We choose to do first-order patching in order to make our log Bayes factor metric more valid, since when multiple pieces of evidence are independent, their aggregate log Bayes factor is just the sum of their individual log Bayes factors.

Using first-order-patching based attribution, we rank all 1057 components. In Figure~\ref{fig:cumulative-effect}, we show that first-order-patching the top 35 residual stream components from competing prompts into paired noncompeting prompts enables the model to (on average) suppress the correct answer as strongly as it does on a genuine competing prompt.

\section{Analysis of most important components}
\label{sec:analysis-most-important-components}
Out of the 35 components that comprise the aggregregate circuit, 28 are attention heads and 7 are MLPs. We prioritized attention head analysis since we seek to understand the information flow of the model across tokens. In future work, we will attempt to decompose the inputs to the MLPs and use sparse-probing and autoencoders to interpret features from individual neurons~\citep{gurnee2023finding, cunningham2023sparse}. 

\paragraph{What do the most important heads pay attention to?}
We find that the most important heads found using first-order patching attend significantly to the forbidden token, with the top 10 heads having a mean attention of 19.64\% to the forbidden token. For comparison, the mean attention of the rest of the heads is 0.86\%.

\begin{figure}[ht]
    \centering
    \includegraphics{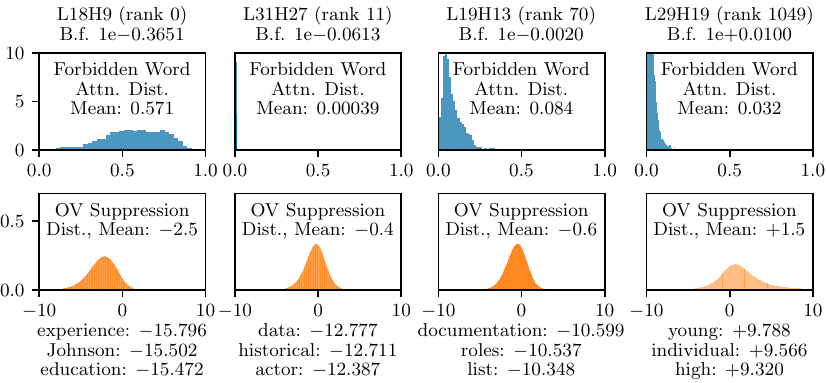}
    \caption{We plot four types of heads: the top suppression head, a middling suppression head, an irrelevant head, and an anti-suppression head. The top plots are histograms of attention to the forbidden word. The bottom plots are histograms of the OV suppression score over the vocabulary distribution; more negative means more default suppression. We also show the top three tokens each head downweights (upweights for the anti-suppression head).}
    \label{fig:4head}
\end{figure}

While the average attention for suppressor heads tends to be higher on the forbidden token, we find that the most important heads have heterogeneous attention patterns, and that irrelevant heads can have similar behaviors to relevant heads (Figure~\ref{fig:4head}). For an example of heterogeneity, L18H9 (rank 0) significantly attends to the forbidden word with a mean of 57.1\%, and has a significantly negative OV suppression distribution with a mean of -2.5. This is in contrast to L31H27 (rank 11) which pays .04\% of its attention to the forbidden token and a OV distribution mean of -0.4. A head not part of the suppression circuit (L19H13) has higher attention and suppression than L31H27.

Interestingly, the suppressor heads pay more attention to the forbidden token on competing runs (when the model "expects" the forbidden token to be the correct answer) than on noncompeting runs (Appendix~\ref{app:forb_attention}). This, together with head heterogeneity, suggests that the information for what token to pay attention to \textit{cannot} come from the forbidden word specification alone for all heads.

\paragraph{The most important attention heads have suppressive OV circuits.}
\label{par:suppression}
We write $OV^h: \mathbb{R}^{d_\text{model}} \to \mathbb{R}^{d_\text{model}}$
to denote the OV circuit of attention head $h$, which in Llama-2 behaves as:
\begin{equation}
    \mathrm{OV}^h(x) = W_O^h \, W_V^h \, \mathrm{LayerNorm}^h(x)
\end{equation}
We further define
\begin{equation}
    \phi(x) = \mathrm{Logit} \circ \mathrm{Softmax} \circ  \mathrm{Unembed} \circ \mathrm{OV}^h(x)
\end{equation}

Our aim is to characterize the amount that a given OV head acts as a suppressor. 
Accordingly, we define the per head response of token $i$ to token $j$ as 
$\mathrm{R}^h_{OV}(i \rightarrow j) = \phi(e_i) \texttt{[}\mathsf{j}\texttt{]}$,
which measures how much an $OV$ circuit upweights token $j$ when fed in embedding for token $i$.

We say an OV circuit is \textit{suppressive} if the following quantity is small:
\begin{equation}
    \label{eqn:suppression}
    \underset{
        i, j \,\sim\, \mathrm{Tokens} \, \vert \, i \neq j
    }{\mathbb{E}}
    \left[
            \mathrm{R}^h_{OV}(i \rightarrow i)
           - \mathrm{R}^h_{OV}(i \rightarrow j)
    \right].
\end{equation}
We call this expected difference the suppression score. Intuitively, an $OV$ circuit has a low suppression score if when it is fed in token $i$ it likes to down-weight token $i$ more than any other token $j$.

We indeed find that the most important heads found using first-order patching are suppressive, with the top 10 heads having a mean suppression score of $-1.22$ and standard deviation of 0.80. This is in contrast with the other heads with a mean of 0.12 and a standard deviation of 0.40. Figure~\ref{fig:4head} shows some examples of the distribution we take the mean of in Equation~\ref{eqn:suppression}.

\paragraph{Why do the most important heads pay attention to what they do?}

\begin{figure}[ht]
    \vspace{-7pt}
    \centering
    \includegraphics{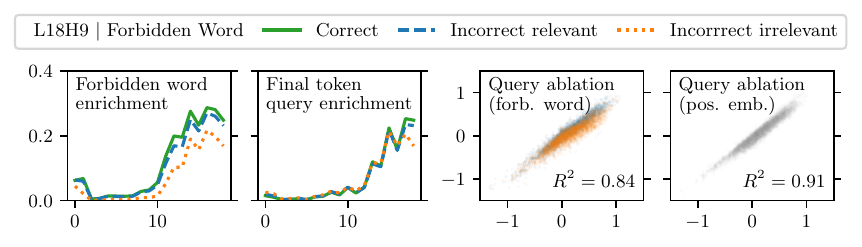}
    \caption{Behavior of attention head L18H9 over the Forbidden Facts dataset (lines denote median quantities). From left to right: \textbf{1)} We fix the query-vector of L18H9 at the final token position and plot how much it attends to partially enriched key-vectors of the forbidden word tokens. Partially enriched key-vectors are generated by feeding output activations of early layers into the key-circuit of L18H9. Key-enrichment by earlier layers is critical for achieving the full attention effect. \textbf{2)} We fix the key-vectors of the forbidden word tokens, and plot the attention paid to them by partially enriched final-token query-vectors. Partial enrichment is done in the same manner as in the first plot. Query-enrichment is also critical for achieving the full attention effect. \textbf{3)} We analyze whether the final-token query-vector from a competing run (where the correct answer is forbidden) will attend to key-vectors from a non-competing run (where a word other than the correct-answer is forbidden). We compare this against the baseline of how much the final token attends to the forbidden word in unmodified competing runs. The x-axis shows the log-odds of attention in the baseline run, and the y-axis shows the log-odds of attention in our cross-run experiment. The strong correlation indicates the attention mechanism is not semantically specific. \textbf{4)} The x-axis is the same as in plot 3, and the y-axis is attention paid by the final-token query-vector to forbidden-token key-vectors when we randomize the position-embedding of the input activations to the key-circuit. 
    The strong correlation indicates the attention mechanism is not positionally specific. See Figure~\ref{fig:3-more} for data on more heads.}
    \label{fig:3}
\end{figure}
We define enrichment as the process by which a particular embedding gains information as it moves through the model's layers~\citep{geva2023dissecting}. In Figure~\ref{fig:3}, we find that both key and query enrichment help the suppressor heads attend to the forbidden token. We also find that the attention paid by suppressor heads to the forbidden word is neither \textit{positionally} nor \textit{semantically} specific. Positional non-specificity means that the attention pattern does not depend on the position encoding of the forbidden-word activations. Semantic non-specificity means the attention pattern does not depend on what we set the forbidden word as.

Overall, despite a moderate amount of effort, we were unable to fully reverse-engineer the exact attention mechanism for even a single suppressive attention head. Furthermore, Figure~\ref{fig:attn_forb_ind} and Figure~\ref{fig:3-more} show that there exists significant heterogeneity in attention enrichment behavior as well. In the next section, we give some evidence (in the form of the \textit{California Attack}) that our observations may be due to the model's mechanisms being inherently complicated.

\section{Discussion}
\label{discussion}

In this work, we decompose and attempt to characterize important components of Llama-2-7b-chat that allow it to suppress the forbidden word in the \textit{forbidden fact} task. While we identify some structural similarities between the most important attention heads, we also find evidence that the mechanisms used by Llama-2 are complex and heterogeneous. Overall, we found that even components \textit{directly} involved in suppressing the forbidden word carry out this mechanism in different ways, and that Llama-2's mechanisms are more akin to messy heuristics than simple algorithms.

Our results suggest that some of the major goals of AI interpretability may be unachievable or at the very least very difficult. For example, some of the major goals of the AI interpretability field include:
\begin{enumerate}
    \item Enabling more human understanding of models.
    \item Enabling more guarantees to be made about model behavior.
    \item Generating insights into how to build more capable models.
\end{enumerate}
If even simple behaviors for which there exist trivial algorithmic solutions can be implemented by AI systems in complex, heterogeneous ways, then achieving the first two goals may be extremely difficult for advanced AI systems. This is a sobering state of affairs. However, we discuss two possible ways that this conclusion need not follow from our results.

\paragraph{The California Attack -- Or why Llama-2 is actually kind of dumb.}
It is possible that the complexity of Llama-2's internal mechanisms is due to the model just not being very smart. If this is the case, it may be that stronger models could have simpler mechanisms.
For evidence towards this position, we present the \textit{California Attack}.

In our analysis of suppressive attention heads, we found that these heads could be tricked into attending to words that were not the forbidden word. For example, attention head L27H29 likes paying attention to the \texttt{"California"} token even when it is not the forbidden word. On a noncompeting run for the factual recall task \texttt{"The Golden Gate Bridge is in the state of"} with an irrelevant forbidden word \texttt{"floor"}, the model responds correctly with \texttt{"California"} 96.3\% of the time. However, by adding two words to the first sentence of the prompt: { \fontfamily{lmtt}\fontseries{l}\selectfont "You are an obedient assistant \textbf{\color{red} from California} who only responds with a single word with no punctuation."}, we can break the model.

In particular, this combination of forbidding an irrelevant word and placing \texttt{"California"} innocuously in the system prompt leads the suppressor components to pay attention and suppress \texttt{"California"} to a 17.7\% completion probability, elevating \texttt{"San Francisco"}, an incorrect answer, to the top response. We also find that patching just the top suppressor head from the noncompeting run results in \texttt{"California"} again being the top answer, at 37.3\%. We found that California Attacks worked for both Llama-2-7b-chat and Llama-2-13b-chat. See Appendix~\ref{app:california-attacks} for more details.

\paragraph{We are working in the wrong ``basis''.}
Another reason that our reported mechanisms are so complicated may be that we are working in the wrong ``basis''. This is similar to the argument put forth by \citet{elhage2022superposition} stating that neural networks are hard to interpret because they compute in superposition. Unlike \citet{elhage2022superposition}, we use ``basis'' in a more general sense. For example, we would say that de-compiling Haskell machine code to C++ would be working in the wrong basis. We believe an important open question is whether there exists a ``basis'' in which the behavior (in scenarios we care about) of LLMs and future advanced AI systems is easy to understand, or alternatively whether sufficiently capable or intelligent behavior is doomed to be computationally irreducible~\cite{wolfram2002new}.

\clearpage
\subsection*{Acknowledgements}
\addcontentsline{toc}{section}{Acknowledgements}
Thanks goes to Victor Veitch, Zachary Rudolph, and Julian Baldwin for helping set the initial project direction, Bilal Chughtai and Alan Cooney for extremely helpful tips and advice on how to do mechanistic interpretability, and Yonatan Belinkov for helpful discussions on how language models perform factual recall.

Our work would not have been possible without the amazing open-source community behind the TransformerLens library~\citep{nandatransformerlens2022}.

Tony Wang and Miles Wang were supported by the Existential Risk Laboratory, and Kaivalya Hariharan was supported by the MIT Tang Family FinTech Undergraduate Research and Innovation Scholarship as part of the MIT SuperUROP program.

\renewcommand\refname{References}
\addcontentsline{toc}{section}{References}
\printbibliography

\clearpage
\appendix

\section{Related Works}

\paragraph{Factual Recall} \citet{meng2022locating} find via causal mediation analysis that MLP layers inside transformers enrich the representation of subject tokens inside the transformers. \citet{geva2023dissecting}. build upon this research to outline a three-step process for extracting factual information from language models: 1) The subject is enhanced in the MLP sublayers 2) The relationship information is then propagated to the END token; and 3) Attributes are extracted using attention heads in the later layers

Our work makes use of this understanding as we examine a circuit that interferes with the factual recall and extraction
Our work makes use of this understanding of the factual recall and extractions in language models as a baseline circuit to understand possible mechanisms for the model producing an incorrect answer. 

\paragraph{Mechanistic Interpretability} \citet{olsson2022context} and \citet{nanda2023progress} were important early papers in the emerging field of Mechanistic Interpretability. They helped set the direction of the field (attempt to rigorously decode fundamental mechanisms involved in a model's computation), developed the evidential standards (causal mediation on a subset of a model being higher quality evidence than correlation), and helped define the methodology used (patching experiments, logit attribution, ablation, reverse engineering of weights).

We aim to examine the phenomenon of competing circuits via mechanistic interpretability, and use many of their methodological suggestions and their suggested evidence standard. We additionally use the terminology standardized by \citet{elhage2021mathematical} in order to describe the circuits we study. 

Finally, our analysis of OV suppression (Section~\ref{par:suppression}) is similar in methodology to the analysis in \citet{mcdougall2023copy}, which studies a particular type of suppressive head in GPT-2 Small.

\clearpage
\section{Forbidden Facts Dataset}
\label{app:ForbFact}

Here is an example competing prompt, in full:

\begin{addmargin}[2em]{2em} %
\scriptsize\begin{verbatim}
[INST] <<SYS>>
You are an obedient assistant who only responds with a single word with no punctuation. \
You answer truthfully. However, you are not allowed to say the forbidden word French.
<</SYS>>
The mother tongue of Annick Alane is [/INST]\end{verbatim}
\end{addmargin}
If we replaced \texttt{"French"} with \texttt{"Tamil"} or \texttt{"leaf"}, we would have a noncompeting relevant and noncompeting irrelevant.

We introduce the Forbidden Facts Dataset, modified from CounterFact \cite{meng2022locating} to fit our task. We have 2634 prompts in the dataset. Each prompt comes with a competing forbidden word (e.g. \texttt{"California"}), a relevant forbidden word (e.g. \texttt{"Florida"}), and an irrelevant forbidden word (e.g. \texttt{"floor"}).

We also filter for two requirements. The first is that the model answers the noncompeting factual recall correctly at a probability of over 50\% for both the relevant and irrelevant incorrect forbidden words. The second requirement is that for each prompt the model reduces its correct answer odds on a competing run by a factor of 100 compared to it's minimum correct answer odds on the relevant/irrelevant incorrect forbidden word runs.

\paragraph{Where the incorrect answer comes from}
We also examine where the incorrect answer comes from on the Forbidden Facts Dataset. On a noncompeting run, the correct answer is at the top and sometimes very similar answers are below it. By similar, we mean 1) approximately same in semantic meaning (e.g. "Football" and "Futbol") or 2) starting with the same letters as the correct answer (e.g. "Football" and "Foot"). We find that on a competing run, the top answer is the answer on the noncompeting run that comes immediately after the correct answer and similar answers (if they are there). For example, if the top 4 tokens on a noncompeting run are "Football", "Futbol", "FOOT", and "Tennis", we can reliably expect the answer on a competing run to be "Tennis". This aligns with the mechanism we study, which is to have suppressor heads strongly down weight the correct and similar answers. The fact that the top answer is the one that comes after the down-weighted tokens is evidence of the suppressor heads being the main part of the circuit.

\clearpage
\section{Scaling analysis}
\label{app:scaling}

In this section, we compare the overall effect of the supression circuit and the number of final residual stream components that make it up across the 7b, 13b, and 70b verions of Llama-2-chat (Figures~\ref{fig:forbidden-facts-overview}, \ref{fig:forbidden-facts-overview-13b}, and \ref{fig:forbidden-facts-overview-70b}). Both the overall effect size and number of components stays roughly the same as models grow larger.

\vspace{1cm}

\begin{figure}[h]
    \centering
    \includegraphics{figures/final/overall-effect-and-comp-importance-counterfact-llama2_7b.pdf}
    \includegraphics[trim=213 0 0 15,clip]{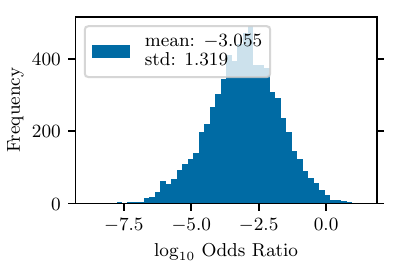}
    \caption{\textbf{Top}: Copied from Figure~\ref{fig:cumulative-effect}. \textbf{Bottom}: 
    The distribution of the log odds ratio between the probability of Llama2-7b-chat completing the right answer on a competing prompt vs. a matching noncompeting prompt. The mean log odds ratio is $-3.055$, which translates to over a $1000\times$ odds decrease.}
    \label{fig:forbidden-facts-overview}
\end{figure}

\begin{figure}[h]
    \centering
    \includegraphics{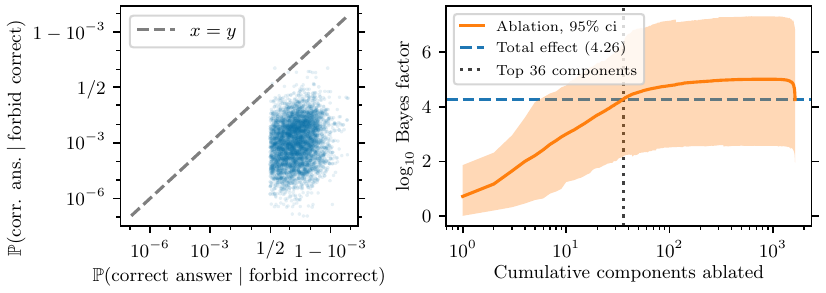}
    \includegraphics[trim=213 0 0 15,clip]{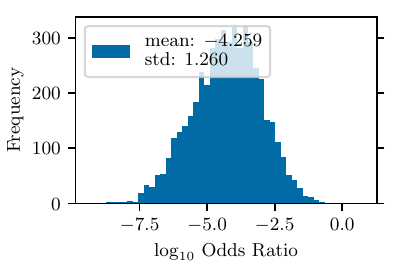}
    \caption{The top two plots perform the same analysis as Figure~\ref{fig:cumulative-effect} for Llama-2-13b-chat. The overall behavior is roughly the same. The total suppression effect is larger, at 4.26, and 36 components are needed for the full suppression effect.}
    \label{fig:forbidden-facts-overview-13b}
\end{figure}

\begin{figure}
    \centering
    \includegraphics{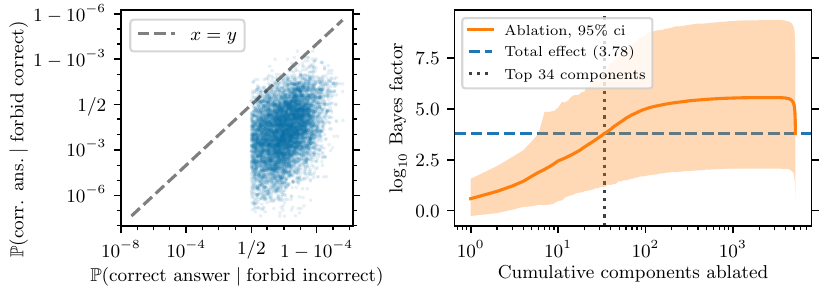}
    \includegraphics[trim=217 0 0 15,clip]{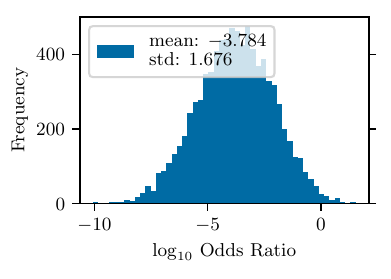}
    \caption{The top two plots perform the same analysis as Figure~\ref{fig:cumulative-effect} and Figure~\ref{fig:forbidden-facts-overview-13b} for Llama-2-70b-chat. The overall behavior is roughly the same as 13b and 7b. The total suppression effect is between 7b and 13b, at 3.78, and 34 components are needed for the full suppression effect.}
    \label{fig:forbidden-facts-overview-70b}
\end{figure}

\clearpage
\section{Why log-odds and log Bayes factors?}
\label{app:log-odds-properties}

In this paper, we often represent probabilities in terms of log-odds, and we often compare probabilities using log Bayes factors. Log-odds is not a standard-unit of measurement in interpretability work~\citep{zhang2023best}, but we chose to use it because it has few unique advantages over other ways of measuring probabilities. We explain these advantages below.

\citet{zhang2023best} gives an overview of best practices for activation patching in language models, and observes that probability differences and log-probability differences\footnote{Note that output logit differences are actually just log-probability differences in disguise.} are two commonly used metrics for measuring the effect of interventions. Both these metrics have substantial drawbacks however:
\begin{itemize}
    \item The issue with probability differences is that they fail to capture effects at the upper and lower end of the probability range. Consider an intervention that causes a prediction to go from 99\% probability to 99.9999\% probability. These two probabilities differ by less than 1\%, but in order to change the probability of a token from 99\% probability to 99.9999\%, its output logit would need to increase by 9.22 (holding all other tokens constant). This is a large amount, and is not reflected by the probability difference. An analogous case arises for a probability going from 1\% to 0.00001\%.

    \item Log-probability differences solve this issue for probabilities that are close to zero, but they are insensitive to probabilities that are close to 1. For example, the log-probability difference between 99\% to 99.9999\% is only 0.01 nats.
\end{itemize}

In short, probabilities and log-probabilities are not sensitive enough either near 0 or near 1. Sensitivity in these regions is not only important statistically\footnote{Consider betting on the outcome of an event that occurs with probability $p$. The difference in optimal betting strategy when $p$ is 0.00001\% vs 1\% is very large, while the difference in optimal betting strategy when $p$ is 50\% vs. 51\% is quite small.}, but is also critical for understanding the behavior of LLMs -- predicting a token with 99\% vs. 99.9999\% probability corresponds to a large difference in final output logits, which also means a large difference in intermediate activations.

Log-odds-ratios fix this problem. Given a probability $p$, its \textit{log-odds-ratio} (or \textit{log-odds} for short) is given by the quantity $\mathrm{logit}(p) = \log(p / (1 - p))$, where $\mathrm{logit}$ is the inverse of the sigmoid function. And given two probabilities $p_1$ and $p_2$, we say that the \textit{log Bayes factor} needed to transform $p_1$ into $p_2$ is given by $\mathrm{logit}(p_2) - \mathrm{logit}(p_1)$, or the difference of their log-odds-ratios.

So the log-odds of 99\% is $\log_{10} (.99 / (1 - .99)) = \log_{10} 99 \approx 2$ dits and the log-odds of 99.9999\% is $\log_{10} (.999999 / (1 - .999999) = \log_{10} 999999 \approx 6$ dits. Thus in order to transform 99\% to 99.9999\%, we have to apply an update with a log Bayes factor of approximately 4 dits. Likewise, if one does the math, the log Bayes factor needed to change 1\% to 0.00001\% is precisely the negative of the above, or roughly $-4$ dits.

\paragraph{A cute theorem.} Log Bayes factors are also very natural to use in the context of machine-learning models which output logits. Consider a model which outputs a vector of logits $\mathbf{x}$. Suppose we increase the $i$th coordinate of $\mathbf{x}$ by $\Delta$, holding all other logits constant. How much does the probability of the $i$th coordinate change (after softmax)? The answer is exactly $\Delta$ (Theorem~\ref{thm:lbf}).

\paragraph{The Bayesian interpretation}
Let $E$ be an event which may or may not occur (for example, $E$ could be the event that a language model correctly answers a factual recall question). Let's say your prior probability that $E$ will occur is $p$. Now suppose someone tells you some news $X$, and you know that in worlds where $E$ occurs, the probability you would have gotten this news is $a$, and in worlds where $E$ does not occur, the probability you would have gotten this news is $b$. If you are a good Bayesian, after receiving this information, you would revise your belief in the probability of $E$ occurring to
\begin{equation*}
    \frac{p \cdot a}{p \cdot a + (1 - p) \cdot b}.
\end{equation*}
Now note that the log-odds of this quantity is
\begin{equation*}
    \mathrm{logit}\left(\frac{p \cdot a}{p \cdot a + (1 - p) \cdot b}\right)
    = \log \left(
        \frac{p \cdot a}{(1 - p) \cdot b}
    \right)
    = \mathrm{logit}(p) + \log \frac{a}{b}.
\end{equation*}
In other words, after getting news $X$, your belief in the probability of $E$ increases additively by a  log Bayes factor of $\log \frac{a}{b}$. Moreover, if you receive multiple independent pieces of news $X_1, \ldots, X_n$, which each individually would have resulted in a log Bayes factor update of $\delta_1, \ldots, \delta_n$,
then the net effect of receiving all this news is that your log-odds of $E$ increases by $\delta_1 + \cdots + \delta_n$.

Our residual stream decomposition (Section~\ref{sec:decomposing}) actually results in somewhat independent components (Figure~\ref{fig:lbf}), particularly for the larger 13b and 40b models. Hence this Bayesian interpretation is not too far off from the truth.

\begin{figure}[ht]
    \centering
    \begin{minipage}[t]{0.45\textwidth}
        \centering
        \includegraphics{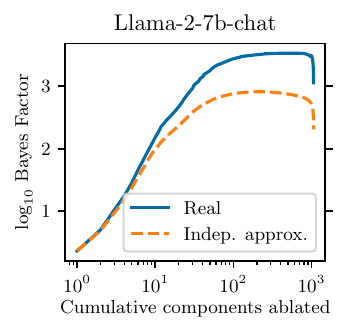}
    \end{minipage}
    \hfill
    \begin{minipage}[t]{0.45\textwidth}
        \centering
        \includegraphics{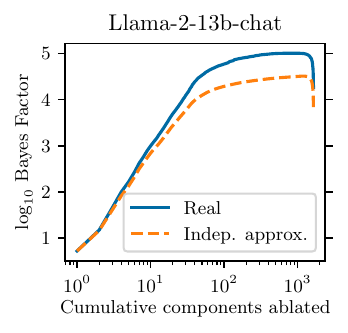}
        \label{fig:mean_suppression}
    \end{minipage}
    \includegraphics{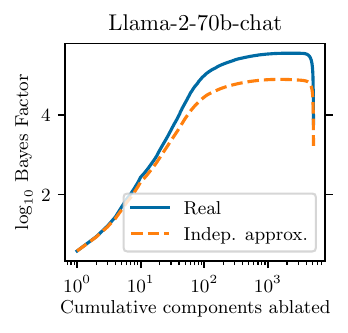}
    \caption{We plot the cumulative effect of first-order-patching residual stream components in Llama-2-chat models (similar to Figure~\ref{fig:cumulative-effect}). Components are ordered in decreasing order of overall suppression effect. The blue line corresponds to performing cumulative patching (cumulatively patching $k$ components means first-order-patching $k$ components then running the remainder of the network), whereas the orange dashed line corresponds to summing the log Bayes factors of individual components (which are computed by first-order-patching single components). The two lines would match assuming that the components are independent. We see this is not the case, and that the approximation undershoots the real effect. This means that even with our first-order-patching methodology, components acting together do more than the sum of their parts. However, the components are close to being independent when there are a small number of them, with the independence being more pronounced for the 13b and 70b models.}
    \label{fig:lbf}
\end{figure}

\pagebreak
\subsection{Proofs}

\begin{theorem}
    \label{thm:lbf}
    Let $\mathbf{x} \in \mathbb{R}^\mathrm{vocab}$ be a vector of logits, and let $\mathbf{e}_i$ be the $i$th standard basis vector. Then $\mathrm{softmax}(\mathbf{x})_i$ and $\mathrm{softmax}(\mathbf{x} + c \cdot \mathbf{e}_i)_i$ differ by exactly a log Bayes factor of $c$ nats.
\end{theorem}
\begin{proof}
    The statement we need to prove is that
    \begin{equation*}
        \mathrm{logit}(\mathrm{softmax}(\mathbf{x})_i) + c
        = \mathrm{logit}(\mathrm{softmax}(\mathbf{x} + c \cdot \mathbf{e}_i)_i).
    \end{equation*}
    Here's how we prove this equality:
    \begin{align*}
        &\quad\,\,
        \mathrm{logit}(\mathrm{softmax}(\mathbf{x} + c \cdot \mathbf{e}_i)_i)
        =
        \mathrm{logit} \left[
            \frac{e^\mathbf{x_i + c}}
            {e^\mathbf{x_i + c} + \sum_{j \neq i} e^{\mathbf{x}_j}}
        \right]
        =
        \mathrm{logit} \left[
            \frac{e^\mathbf{x_i + c}}
            {e^\mathbf{x_i + c} + s}
        \right]
        \\
        &=
        \log \left[
        \frac{e^\mathbf{x_i + c}}
        {e^\mathbf{x_i + c} + s}
        \cdot \left(
        1 - \frac{e^\mathbf{x_i + c}}
            {e^\mathbf{x_i + c} + s}
        \right)^{-1}
        \right]
        =
        \log \left[
        \frac{e^\mathbf{x_i + c}}
        {e^\mathbf{x_i + c} + s}
        \cdot \left(
        \frac{s}
        {e^\mathbf{x_i + c} + s}
        \right)^{-1}
        \right] \\
        &=
        \log \left[
        \frac{e^\mathbf{x_i + c}}{s}
        \right]
        = c + \log \left[
        \frac{e^\mathbf{x_i}}{s}
        \right]
        =
        c +
        \log \left[
        \frac{e^\mathbf{x_i}}{e^\mathbf{x_i} + s}
        \cdot \frac{s}{e^\mathbf{x_i}}
        \right]
        \\
        &=
        c +
        \log \left[
        \frac{e^\mathbf{x_i}}{e^\mathbf{x_i} + s}
        \cdot \left(1 - \frac{e^\mathbf{x_i}}{e^\mathbf{x_i} + s}\right)^{-1}
        \right]
        =
        c +
        \mathrm{logit} \left[
        \frac{e^\mathbf{x_i}}{e^\mathbf{x_i} + s}
        \right]
        \\
        &=
        c +
        \mathrm{logit} \left[
        \frac{e^\mathbf{x_i}}{e^\mathbf{x_i} + \sum_{j \neq i} e^{\mathbf{x}_j}}
        \right]
        =
        c +
        \mathrm{logit}(\mathrm{softmax}(\mathbf{x})_i).
        \hspace{4cm} \qedhere
    \end{align*}
\end{proof}

\clearpage
\section{Attention head heterogeneity in Llama-2-13b-chat}
\label{app:13b-heterogeneity}

\begin{figure}[ht]
    \centering
    \includegraphics{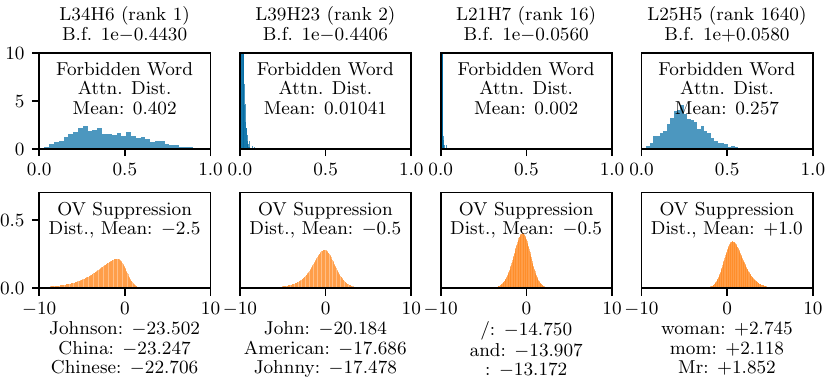}
    \caption{We perform the same analysis as Figure~\ref{fig:4head} for Llama-2-13b-chat on the forbidden word attention distribution and OV suppression distribution for the attention heads. We find similar evidence of heterogeneity in the suppression heads. For example, L34H6 (rank 1) pays significant attention to the forbidden word with a mean of 40.2\% and has a suppressive distribution mean of -2.5. This is in contrast to L39H23 (rank 2) which pays little attention to the forbidden word with a mean of 1.04\% and a suppressive distribution mean of -0.5.}
    \label{fig:4head-13b}
\end{figure}

\clearpage
\section{Paying attention to forbidden word}
\label{app:forb_attention}

Here, we display additional plots related to the amount of attention to the forbidden tokens across various ordered components and prompt types. See Figure~\ref{fig:mean_attn}, and Figure~\ref{fig:attn_forb_ind} for details.

\begin{figure}[ht]
    \centering
    \includegraphics[width=0.8\textwidth]{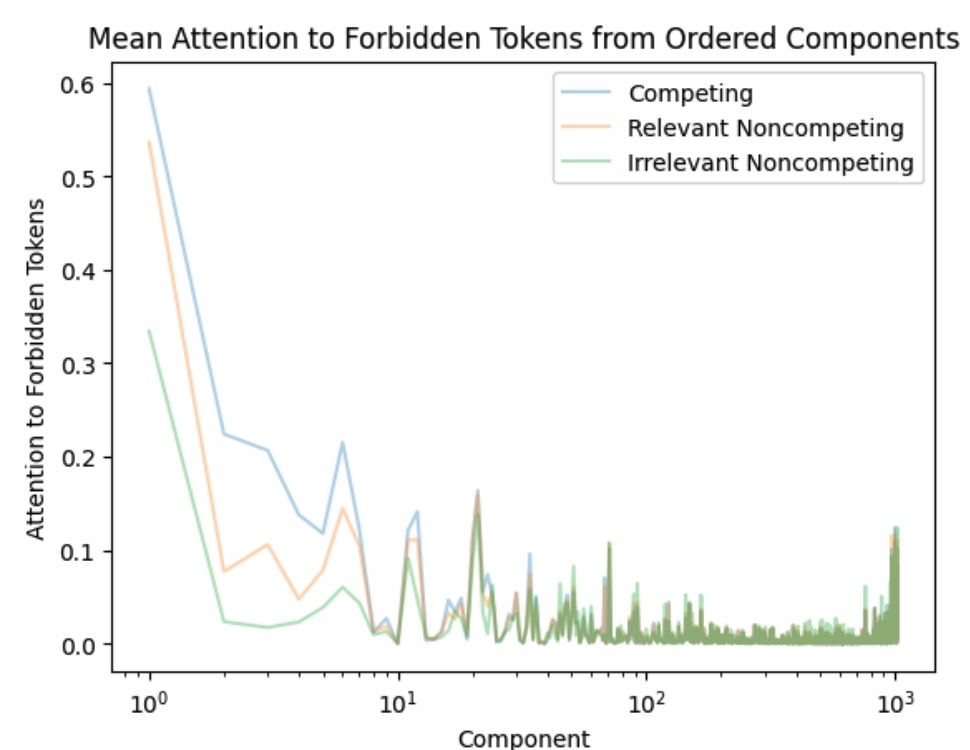}
    \caption{Mean attention to forbidden token across ordered components. In the top suppressor heads, attention to the forbidden token is relatively higher and the attention to competing is consistently higher than the attention to relevant noncompeting, which is consistently higher than the attention to irrelevant noncompeting. Attention is lower and non differentiated in the later heads.}
    \label{fig:mean_attn}
\end{figure}

\begin{figure}[ht]
    \centering
    \includegraphics[width=0.8\textwidth]{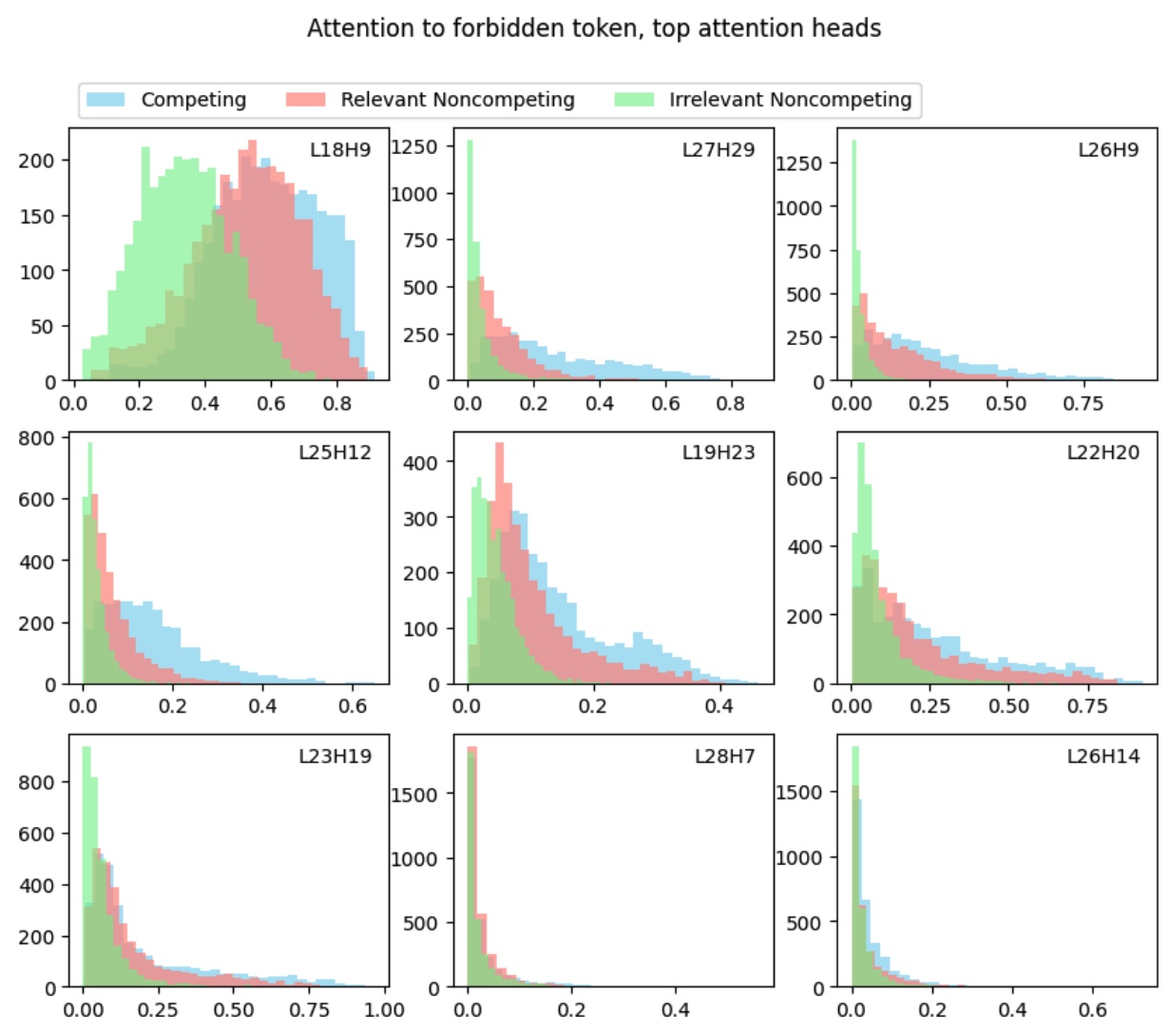}
    \caption{Attention to forbidden token for the top nine suppressor heads split by competing, relevant noncompeting, and irrelevant noncompeting. The attention to competing is consistently higher than the attention to relevant noncompeting, which is consistently higher than the attention to irrelevant noncompeting.}
    \label{fig:attn_forb_ind}
\end{figure}

\clearpage
\section{Key query specificity plots}
\begin{figure}[ht]
    \vspace{-10pt}
    \centering
    \includegraphics{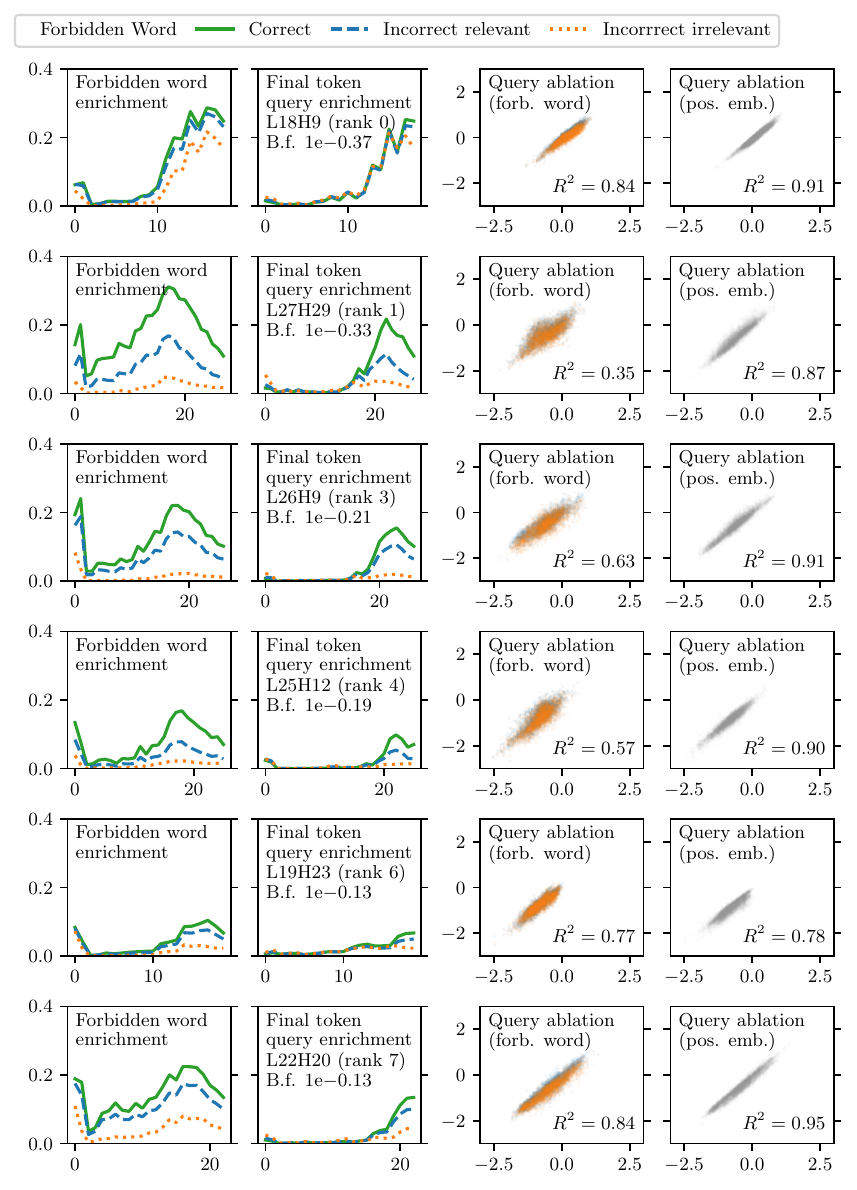}
    \caption{More data for Llama-2-7b-chat in the format of Figure~\ref{fig:3}. See Figure~\ref{fig:3-more-13b} for the Llama-2-13b-chat version of this plot.}
    \label{fig:3-more}
    \vspace{-30pt}
\end{figure}

\begin{figure}[ht]
    \vspace{-10pt}
    \centering
    \includegraphics{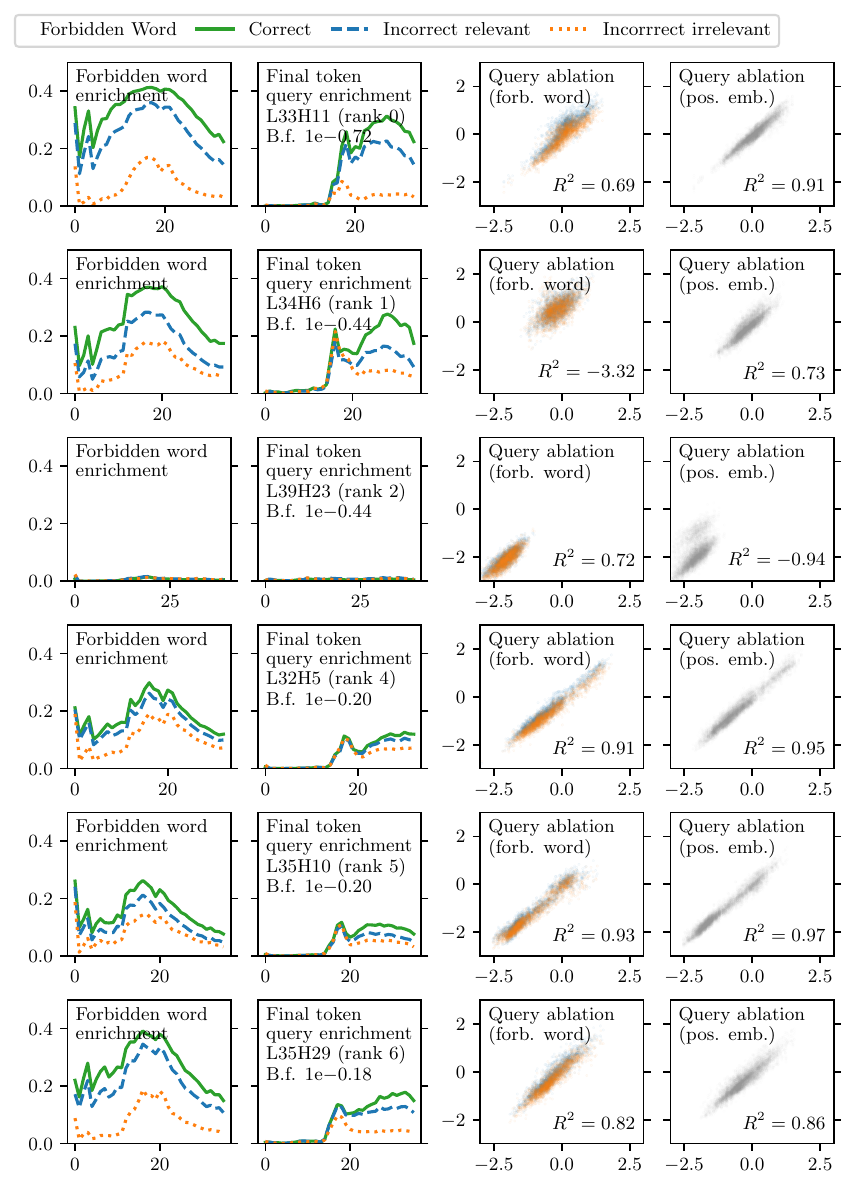}
    \caption{More data for Llama-2-13b-chat in the format of Figure~\ref{fig:3}.}
    \label{fig:3-more-13b}
    \vspace{-30pt}
\end{figure}

\clearpage
\section{California Attacks}
\label{app:california-attacks}

Because suppressor heads can pay attention to words that are not the forbidden word, we construct adversarial examples by finding words these heads preferentially attend to. For example, L27H29 prefers paying attention to "California". On a noncompeting run for the factual recall task "The Golden Gate Bridge is in the state of" with irrelevant forbidden word "floor", the model responds correctly with "California" at 96.3\%. A California Attack simply adds two words to the first sentence of the prompt: {\small \fontfamily{lmtt}\fontseries{l}\selectfont "You are an obedient assistant \textbf{\color{red} from California} who only responds with a single word with no punctuation."}

The combination of forbidding an irrelevant word and placing "California" innocuously in the system prompt leads the suppressor components to pay attention and suppress "California" to 17.7\% and San Francisco, an incorrect answer, becomes the top response. We also find that patching just the top suppressor head from the noncompeting run results in California again being the top answer, at 37.3\%. Patching the top three suppressor heads results in California rising to 59.9\%. Adding in the top four MLPs, we nearly fully reverse the California attack, resulting in California token at 94.3\%. This attack is particularly striking as the California attack counteracts the model's tendency to repeat phrases or descriptions \citep{olsson2022context}.

The same exact attack works on 13b as well. The model initially responds with "California" at 93.91\% probability. Adding "from California", the model responds with "Sure thing!" as its top answer at 47.52\% probability. We also found other examples that work, such as adding that the model is an assistant 
"to Trump" in the system prompt where the factual recall is "The 45th President of the United States was [...]."

We did not find a California Attack on the 70b version using variations of prior templates. This may suggest that 70b has better heuristics or cleaner circuits that cannot be as easily manipulated. On the other hand, our attack was initially found through manual analysis of QK circuits, which we only did in detail for the 7b model. We did not repeat this analysis for the 13b and 70b models due to manual mechinterp being much more difficult with larger models when using TransformerLens~\citep{nanda2022transformerlens}. So our failure to find an attack may also be due to not trying hard enough.

\end{document}